\documentclass[twocolumn]{article}
\RequirePackage{fancyhdr}
\RequirePackage{color}
\RequirePackage{algorithm}
\RequirePackage{algorithmic}
\usepackage{times}
\usepackage{graphicx}
\usepackage{latexsym}
\usepackage{subfigure} 
\usepackage[numbers]{natbib}
\usepackage[hidelinks]{hyperref}
\usepackage{subfigure}

\newcommand{\subfloat}{\subfigure}

\usepackage{amssymb}
\usepackage{amsthm}
\usepackage{amsmath}
\usepackage{tabularx}
\usepackage[capitalise]{cleveref}
\usepackage{pgfplots}
\usepackage{pgfplotstable}
\usepgfplotslibrary{colormaps}
\usepgfplotslibrary{units}
\usepgfplotslibrary{external}
\tikzsetexternalprefix{./}
\tikzset{
    external/system call={%
    xelatex \tikzexternalcheckshellescape
    -halt-on-error -interaction=batchmode --shell-escape
    -jobname "\image" "\texsource"}}
\newtheorem{thm}{Theorem}

\newcommand{\diag}{\ensuremath{\text{diag}}}

\renewcommand{\Pr}{P}

\newcommand{\unit}[1]{\ensuremath{\, \mathrm{#1}}}

\definecolor{darkred}{rgb}{0.6,0,0}
\definecolor{darkblue}{rgb}{0,0,0.6}
\definecolor{darkgreen}{rgb}{0,0.6,0}
\definecolor{lightred}{rgb}{1,0.6,0.6}
\definecolor{lightblue}{rgb}{0.6,0.6,1}
\definecolor{lightgreen}{rgb}{0.6,1,0.6}

\title{{\large Technical Report}\\Transductive Learning for \\ Multi-Task Copula Processes}

\author{Markus Schneider\footnote{Ravensburg-Weingarten University of Applied Sciences, formerly affiliated with Australian Centre for Field Robotics \newline email: m.schneider@acfr.usyd.edu.au}
\and Fabio Ramos\footnote{Australian Centre for Field Robotics, School of Information Technologies, The University of Sydney, Australia,\newline email: f.ramos@acfr.usyd.edu.au}
}

 \twocolumn
\def\addcontentsline#1#2#3{}
\def\toptitlebar{\hrule height1pt \vskip .25in} 
\def\bottomtitlebar{\vskip .22in \hrule height1pt \vskip .3in} 

\def\section{\@startsection{section}{1}{\z@}{-0.12in}{0.02in}
             {\large\bf\raggedright}}
\def\subsection{\@startsection{subsection}{2}{\z@}{-0.10in}{0.01in}
                {\normalsize\bf\raggedright}}
\def\subsubsection{\@startsection{subsubsection}{3}{\z@}{-0.08in}{0.01in}
                {\normalsize\sc\raggedright}}
\def\paragraph{\@startsection{paragraph}{4}{\z@}{1.5ex plus
  0.5ex minus .2ex}{-1em}{\normalsize\bf}}
\def\subparagraph{\@startsection{subparagraph}{5}{\z@}{1.5ex plus
  0.5ex minus .2ex}{-1em}{\normalsize\bf}}
\skip\footins 9pt 
\def\footnoterule{\kern-3pt \hrule width 0.8in \kern 2.6pt } 
  \labelwidth\leftmargini\advance\labelwidth-\labelsep \labelsep 5pt 
\begin{document}

\twocolumn[
  {\center\baselineskip 18pt\toptitlebar{\Large{\bf Transductive Learning for Multi-Task Copula Processes}}
  \\ -Technical Report-
  \bottomtitlebar}
  {\bf Markus Schneider} \hfill {\sc m.schneider@acfr.usyd.edu.au}%
  \par\vskip 0.03in 
  Institute for Artificial Intelligence\\
  Ravensburg-Weingarten University of Applied Sciences, Germany\\
  formerly affiliated with\\
  Australian Centre for Field Robotics\\
  The University of Sydney, Australia 
  \vskip 0.10in
  {\bf Fabio Ramos} \hfill {\sc f.ramos@acfr.usyd.edu.au}%
  \par\vskip 0.03in Australian Centre for Field Robotics,\\
  School of Information Technologies\\
  The University of Sydney, Australia 
  
  \vskip 0.3in
]
\thispagestyle{empty}
\bibliographystyle{plain}

\begin{abstract} 

We tackle the problem of multi-task learning with copula process. Multivariable prediction in spatial and spatial-temporal processes such as natural resource estimation and pollution monitoring have been typically addressed using techniques based on Gaussian processes and co-Kriging. While the Gaussian prior assumption is convenient from analytical and computational perspectives, nature is dominated by non-Gaussian likelihoods. Copula processes are an elegant and flexible solution to handle various non-Gaussian likelihoods by capturing the dependence structure of random variables with cumulative distribution functions rather than their marginals.  We show how multi-task learning for copula processes can be used to improve multivariable prediction for problems where the simple Gaussianity prior assumption does not hold. Then, we present a transductive approximation for multi-task learning and derive analytical expressions for the copula process model. The approach is evaluated and compared to other techniques in one artificial dataset and two publicly available datasets for natural resource estimation and concrete slump prediction. 

\end{abstract} 

\section{Introduction}
\label{sec:intro}

\emph{Multi-task learning} is valuable in many areas of research such as spatial-temporal modeling, environmental sciences, numerical optimization and data fusion. In these problems it is advantageous to predict more than one quantity at a time (in contrast to single-task learning) to exploit inter-dependencies. Kernel-based algorithms achieve this by the use of an appropriate \emph{multi-task} kernel. \emph{Gaussian process} (GP) \cite{Rasmussen:2006bj} based regression, as a simple and fully probabilistic model, is often the tool of choice for such problems. The GP framework supports an easy specification of a regression prior using a mean function and a kernel and delivers closed form solutions at inference stage for the predictive mean and variance. However, in many cases the assumption of a Gaussian likelihood of the data is incorrect, but accepted because of the mathematical elegance of the GP framework and the lack of alternatives with comparable performance.

\emph{Copulas}, with roots in statistics~\citep{Sklar:1996ys} are models that separate the dependence structure of two or more random variables from their marginal distribution, thus possessing the flexibility of using a different probability distribution function for each variable. Informally, they perform a transformation that maps each variable through its cumulative distribution function (cdf) to the unit interval and captures the dependence between the variables using a coupling term. This methodology can handle complex joint distributions between random variables offering tractable solutions for conditional and marginal operations. Copula distributions can be extended to stochastic processes \citep{Jaimungal:2009mi} with the help of kernels. It can be shown that the Gaussian process model is just a special case of the copula process if Gaussian marginal distributions are used. This makes copula processes an appealing replacement for GPs in cases where the Gaussian assumption is not appropriate. In this paper, we address the computational costs of copula processes, which make their application to multi-task problems difficult. We introduce a general \emph{transductive} approximation and provide analytical expressions for multi-task copula processes. 

\subsection{Related Work}

Copula processes are relatively new in machine learning. After the fundamental work on copula processes \citep{Jaimungal:2009mi}, an alternative to GARCH models for finance applications using a copula based stochastic volatility model was proposed \citep{Wilson:2010lq}. It can also be shown \citep{Wauthier:2010qa} that a heavy-tailed process, derived from copula theory, can provide robustness against outliers in the data. In geostatistics the copula process is called \emph{copula based Kriging estimator} \citep{Kazianka:2010rr} and had been introduced as a possible improvement over Gaussian random fields.

Multi-task learning is a more general form of co-Kriging where predictions for multiple quantities are made at the same time. Several different methods had been proposed for multi-task Gaussian processes: The task dependence can be introduced with shared hyper parameters \citep{Minka:1997vn} or an appropriate prior on the covariance matrix as, for example an inverse-Wishart distribution \citep{Yu:2005zr}. It is also possible to construct new kernel functions \citep{Boyle:2005ly,Boyle:2005rt} if the GP is considered as a convolution of a continuous white noise process with a smoothing kernel \citep{Higdon:1998fk}.

The \emph{Bayesian committee machine} (BCM) \citep{Tresp:2000ba} is a \emph{local} approximation for general probabilistic learning algorithms. The BCM divides the complete training data set into smaller subsets, which are trained individually and then re-combined again for predictions. It belongs to the family of transductive algorithms because the predictive distribution depends on the number and location of the query points. The algorithm is very popular for Gaussian process regression with its high demand for memory and computational time, however it is in general not straight forward how to divide the training data set. Other \emph{global} methods for Gaussian processes such as \emph{Deterministic Training Conditional} \citep{Seeger:2003kx}, \emph{Fully Independent Training Conditional} \citep{Snelson:2006fk} and \emph{Partially Independent Training Conditional} \citep{Candela:2005fk} approximation generate a sparse covariance matrix by identifying and using only a representative subset of the training data while ignoring or approximating the other samples in the set. A framework and an excellent discussion about these methods can be found in \citep{Candela:2005fk} and we will show later how our approximation can fit into it. Recently a structured noise covariance that is independent of the inputs but captures residual correlation between tasks was proposed in \citep{Rakitsch:2013}. Gaussian process regression networks \citep{wilson12icml} model related tasks with an adaptive mixture of GPs.

The novelty of this work lies in the derivation of a transductive approximation for Bayesian multi-task problems. We show how the computational complexity can be handled for a large number of tasks, which normally would grow significantly, the more variables are estimated simultaneously. Furthermore we show the practical consequences on Gaussian copula processes with a multi-task kernel obtained through process convolution.

\section{Copula Processes}

Copulas are a statistical framework to decompose a joint distribution of random variables $H(y_1, \dotsc, y_n)$ into their univariate marginal cumulative distribution functions (cdf) $F_1, \dotsc, F_n$ and a coupling term, the actual copula. Hereby each random variable gets mapped though its marginal distribution into the $[0,1]$ interval, called probability integral transformation $u_i = F_i(y_i)$. The task is then to find a copula $\mathcal{C}$, such that
\begin{align}
	H(y_1, \dotsc, y_n) &= \mathcal{C} \left( F_1(y_1), \dotsc, F_n(y_n) \right) \label{eq:sklar} \\
						&= \mathcal{C} \left(u_1, \dotsc, u_n\right). \nonumber	
\end{align}
The distribution $\mathcal{C}$ has to meet certain requirements \citep{Kolev:2006ij,Nelsen:1998qy}, but it can be proven \citep{Sklar:1996ys}, that a decomposition as in \cref{eq:sklar} exists for every joint distribution $H$.

Selecting one specific copula, it is also possible to create a huge set of different multivariate distributions by varying different marginal distribution functions. A copula with favorable analytical properties is the \emph{Gaussian copula} $C$ which can be constructed from the multivariate Gaussian cdf $\Phi_{\mu,\Gamma}$ with mean $\mu$ and covariance matrix $\Gamma$ as
\begin{align*}
	C_{\mu,\Gamma}(u) &= \Phi_{\mu,\Gamma}\left(\Phi_{\mu_1,\Gamma_{11}^{1/2}}^{-1}(u_1), \dotsc,\Phi_{\mu_n,\Gamma_{nn}^{1/2}}^{-1}(u_n) \right), 
\end{align*}
where $u = (u_1, \dotsc, u_n)$, $\Phi_{\mu_i,\Gamma_{ii}^{1/2}}$ is the $i$th univariate Gaussian cdf with $\Gamma_{ii}^{1/2}$ as the square root of the $i$-th diagonal element of matrix $\Gamma$ and $\mu_i$ as the $i$-th element of vector $\mu$. Its density $c$ can be derived as
\begin{align}
	c_{\mu,\Gamma}(u) &= \frac{\mathcal{N}_{\mu,\Gamma}\left(\Phi_{\mu_1,\Gamma_{11}^{1/2}}^{-1}(u_1),\dotsc,\Phi_{\mu_n,\Gamma_{nn}^{1/2}}^{-1}(u_n) \right)}{\prod_{i=1}^n \mathcal{N}_{\mu_i,\Gamma_{ii}^{1/2}}\left(\Phi_{\mu_i,\Gamma_{ii}^{1/2}}^{-1}\left( u_i \right)\right)}, \label{eq:copula_density}
\end{align}
where we used $\mathcal{N}$, the Gaussian density, as the derivative of $\Phi$. A \emph{Gaussian copula process} can be created \citep{Jaimungal:2009mi} if the Gaussian distribution $\Phi_{\mu,\Gamma}$ gets replaced by though a Gaussian process.

A Gaussian process \citep{Rasmussen:2006bj} $\{Z_x\}$ is a collection of Gaussian random variables indexed by $x$ with mean function $m(x)$ and a positive definite kernel $k(x,x')$. 
For a finite set of input locations $X = (x_1, \dotsc, x_n)$ and corresponding outputs $y = (y_1, \dotsc, y_n)$ we write the density of the finite dimensional subset of $\{Z_x\}$ as $h_{Z_X}(y) = p(Z_X) = \mathcal{N}_{\mu(X),K(X,X)}(y)$ and mean the multivariate normal density with kernel matrix $[K(X,X)]_{i,j} = k(x_i, x_j)$ and mean vector $[\mu(X)]_i = m(x_i)$. To ease notation we follow \citep{Wauthier:2010qa} and assume the mean function to be constant zero and $k(x,x) = \gamma^2, \forall x$ from now on. 

Given these notations we can construct a Gaussian copula process $\{Y_x\}$ with marginal distribution function $F_1, \dotsc, F_n$ as
\begin{align*}
C_{0,K(X,X)}(F_1(y_1),\dotsc,F_n(y_n))
\end{align*}
\begin{align}
	p(Y_X) &= c_{0,K(X,X)}(F_1(y_1),\dotsc,F_n(y_n)) \cdot \prod_{i=1}^n \frac{\partial F_{i}(y_i)}{\partial y_i}. \label{eq:copula_margin_density}
\end{align}
Notice that we can recover the warped Gaussian process \citep{snelson2004warped} if we use $\Phi_{0,\gamma}^{-1}(F_{i}(y_i))$ as warping functions and can also get the Gaussian process as a special case if we set $F_i = \Phi_{0,\gamma}$.

The predictive distribution $p(Y_{X^\ast} \vert Y_{X})$ for $y^\ast = (y_1^\ast,\dotsc,y_m\ast)$ at $X^\ast = (X_1^\ast,\dotsc,X_m^\ast)$ with margins $F_{1}^\ast,\dotsc,F_{m}^\ast$ can be obtained as
\begin{align}
& p(Y_{X^\ast} \vert Y_{X}) = c_{\hat{\mu},\hat{\Gamma}}(F_1^\ast(y^\ast_1),\dotsc,F^\ast_m(y^\ast_m)) \cdot \prod_{i=1}^m \frac{\partial F^\ast_{i}(y^\ast_i)}{\partial y^\ast_i} \label{eq:gcpdensity}\\
\hat{\mu} 		&= K(X,X^\ast)^T K(X,X)^{-1} w \nonumber \\ 
\hat{\Gamma}	&= K(X^\ast,X^\ast) - K(X,X^\ast)^T K(X,X)^{-1} K(X,X^\ast)\nonumber 
\end{align}
and $w_i = \Phi_{0,\gamma}^{-1}(F_{i}(y_i))$.

\subsection{Making Predictions}

In the inference step we normally want to provide an estimate, a single number with uncertainty bounds rather than a full predictive distribution. In machine learning this is often the mean and the variance. This can be problematic for the Gaussian copula process, since (depending on the marginal distribution function) these quantities may not exist. Furthermore, since the predictive distribution can be asymmetric, the variance may not be informative enough. Hence we suggest to calculate the \emph{median} and the \emph{quantiles} of the predictive distribution and provide the expressions next. In order to get a quantile $Q(p)$ at input $X^\ast_i$ we use
\begin{align*}
	Q(p) = F_{i}^{*-1}\left(\Phi_{\mu_i,\Gamma_{ii}^{1/2}}\left( \Phi^{-1}_{\hat{\mu}_i,\hat{\Gamma}_{ii}}(p)   \right)\right),
\end{align*}
where $\hat{\mu}_i$, $\hat{\Gamma}_{ii}$ are the $i$-th entry from $\hat{\mu}$, $\hat{\Gamma}$ as in \cref{eq:gcpdensity} and $F_{i}^{*-1}$ is the quantile distribution of the corresponding cdf\footnote{Notice, that the dependence on $X^\ast_i$, $X$ and $y$ is introduced implicitly by these variables.}. The median as the $0.5$ quantile is then, using the equation above, given by 
\begin{align*}
	Q(0.5) &= F_{i}^{*-1}\left(\Phi_{\mu_i,\Gamma_{ii}^{1/2}}\left( \Phi^{-1}_{\hat{\mu}_i,\hat{\Gamma}_{ii}}(0.5)   \right)\right)\\
	& = F_{i}^{*-1}\left(\Phi_{\mu_i,\Gamma_{ii}^{1/2}}\left( \hat{\mu}_i   \right)\right),
\end{align*}
since the $0.5$ quantile, the median, of a Gaussian distribution is its mean.

\subsection{Multi-Task Copula Processes}
\label{sec:multi-task}

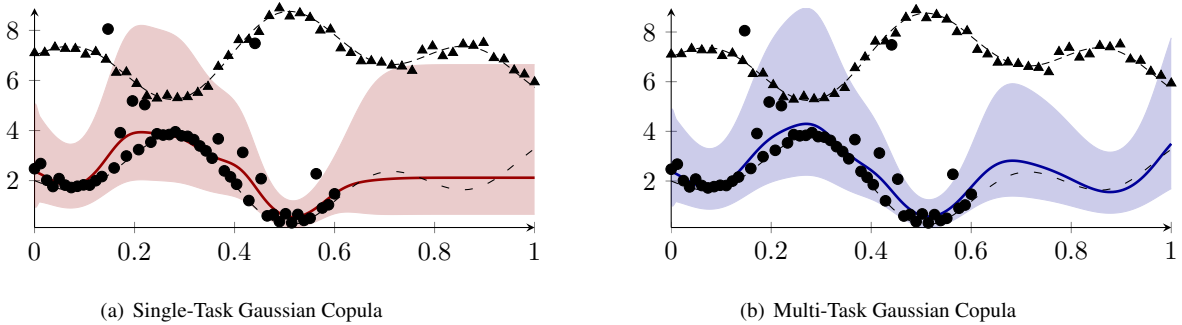
\begin{figure*}[t] 
	\centering
	\subfloat[Single-Task Gaussian Copula]{\label{fig:toy_multi_task_a}
	\begin{tikzpicture}
		\begin{axis}[width=.48\textwidth, height=4.5cm, axis x line=bottom, axis y line=left,]
		\def\filename{mtcopula_toy_stpredict.txt}
		
		\pgfplotstableread{\filename}\ubtable
		\pgfplotstablecreatecol[create col/expr={\thisrow{3}}]{bound}\ubtable
		\pgfplotstablesort[sort cmp=float <]{\ubtable}{\ubtable}
		
		\pgfplotstableread{\filename}\table
		\pgfplotstablecreatecol[create col/expr={\thisrow{2}}]{bound}\table
		\pgfplotstablesort[sort cmp=float >]{\table}{\table}
		
		\pgfplotstablevertcat{\table}{\ubtable}
		\addplot[smooth,darkred,line width=1pt] table[x index=0,y index=1,] {\filename};		
		\addplot[smooth,darkred,fill=darkred,opacity=0.2] table[x index=0, y = bound]{\table};
		\addplot gnuplot[loosely dashed, no marks,black,domain=0:1,samples=100]{-sin(16*x)+sin(8*x)+2};		
		\addplot gnuplot[densely dashed, no marks,black,domain=0:1,samples=100]{+sin(16*x)-sin(8*x)+7};			
		\addplot[only marks,mark=*,mark options={fill=black}] table[x index=0, y index=2] {mtcopula_toy_dataset.txt};	
		\addplot[only marks,mark=triangle*, mark options={fill=black}] table[x index=1, y index=3] {mtcopula_toy_dataset.txt};			
				
		\end{axis}
	\end{tikzpicture}
	}
	\subfloat[Multi-Task Gaussian Copula]{\label{fig:toy_multi_task_b}
	\begin{tikzpicture}
		\begin{axis}[width=.48\textwidth, height=4.5cm, axis x line=bottom, axis y line=left,]
		\def\filename{mtcopula_toy_mtpredict.txt}
		
		\pgfplotstableread{\filename}\ubtable
		\pgfplotstablecreatecol[create col/expr={\thisrow{3}}]{bound}\ubtable
		\pgfplotstablesort[sort cmp=float <]{\ubtable}{\ubtable}
		
		\pgfplotstableread{\filename}\table
		\pgfplotstablecreatecol[create col/expr={\thisrow{2}}]{bound}\table
		\pgfplotstablesort[sort cmp=float >]{\table}{\table}
		
		\pgfplotstablevertcat{\table}{\ubtable}
	
		\addplot[smooth,darkblue,line width=1pt] table[x index=0,y index=1,] {\filename};	
		\addplot[smooth,darkblue,fill=darkblue,opacity=0.2] table[x index=0, y = bound]{\table};
		\addplot gnuplot[loosely dashed, no marks,black,domain=0:1,samples=100]{-sin(16*x)+sin(8*x)+2};		
		\addplot gnuplot[densely dashed,no marks,black,domain=0:1,samples=100]{+sin(16*x)-sin(8*x)+7};		
		\addplot[only marks,mark=*,mark options={fill=black}] table[x index=0, y index=2] {mtcopula_toy_dataset.txt};	
		\addplot[only marks,mark=triangle*, mark options={fill=black}] table[x index=1, y index=3] {mtcopula_toy_dataset.txt};	

		\end{axis}
	\end{tikzpicture}
	}
	\caption[Toy Example]{Comparison between single-task (left) and multi-task (right) Gaussian copula processes for an artificial dataset. The dashed lines representing the (latent) true functions with noisy samples marked as circles and triangles. The observations for the primary task (circles) are corrupted noisy versions of the true function. The noise was sampled from a generalized extreme value distribution. For the observations of the co-task (triangles) we simply added a zero mean Gaussian noise. The solid lines are the predictive median, whereas the shaded region is the area in between the $0.05$ and $0.95$ quantiles.}
	\label{fig:toy_multi_task}
\end{figure*}

In contrast to single-task learning, where the objective is to estimate a scalar valued quantity, the aim of multi-task learning is to estimate more than one variable at a time. The applications of multi-task learning are broad, but very often the estimation of a primary variable of interest can be improved if we take other correlated variables (secondary or co-variables) into account. \cref{fig:toy_multi_task} illustrates this concept. A typical situation is the case where only a small sample set of the primary variable is available, but a larger data set for the secondary variables. This can happen, for example if the primary variable is much more difficult or expensive to estimation and occurs frequently in geology or environmental setting.

The challenge to extend a kernel-based algorithm (such as GPs or the Gaussian copula process) to a multi-task version gets reduced to the problem of defining an appropriate multi-task kernel. Some multi-task kernels are inspired from co-Kriging theory \citep{Wackernagel:2003uq} as, for example the \emph{intrinsic correlation model} (ICM) and \emph{linear model of corregionalization} (LMC). Others are more recent such as the convolutional kernel \citep{Higdon:1998fk}. Given kernels for the individual tasks, the convolutional kernel attempts to find a kernel for all cross-task dependencies such that the resulting kernel matrix is still positive definite. For example, the cross-task kernel between a \emph{squared exponential}, $k(r)=\exp(-r^2/l_{SE}^2)$, and a \emph{Mat\'{e}rn} kernel with smoothness $\nu=3/2$, $k(r)=(1+\sqrt{3}r/l_M)\exp(-\sqrt{3}r/l_M)$ is
\begin{align*}
k^{dq}(x,x') &= k^{dq}(r) = \sqrt{\lambda} \left( \frac{\pi}{2} \right)^{1/4} e^{\lambda^2} 
\Bigg[ 2 \cosh \left( \frac{\sqrt{3}r}{l_M} \right) \\
&- e^{\frac{\sqrt{3}r}{l_M}} \text{erf} \left( \lambda + \frac{r}{l_{SE}} \right) 
- e^{\frac{\sqrt{3}r}{l_M}} \text{erf} \left( \lambda - \frac{r}{l_{SE}} \right) 
\Bigg]
\end{align*}
\noindent where $\lambda = \frac{\sqrt{3}}{2} \frac{l_{SE}}{l_M}$, 
$\text{erf}(x) = \frac{2}{\sqrt{\pi}} \int_0^x e^{-z^2} \, \mathrm{d} z$, $r = ||x-x'||^2$ and $l_{SE}$, $l_M$ are the length scales for the squared exponential (task $d$) and the Mat\'{e}rn (task $q$) respectively \citep{Melkumyan:2011pd}.

We do not have to change anything on the equations for the copula processes, but use a multi-task kernel instead of the ordinary kernel function. We can then merge the inputs and outputs from different tasks into the sets $X$ and $y$ respectively and do the same for test inputs $X^\ast$ and outputs $y^\ast$.

\subsection{Parameter Estimation}

The copula process model is not entirely parameter free as the kernel and the univariate marginal distributions are usually parameterized in some way. We will denote the set of all these parameters with $\theta$. One of the advantages of Bayesian methods is the ability to estimate such parameters from data rather than using cross-validation. We follow the standard procedure and using a maximum likelihood approach. More specifically, we are going to minimize the negative log-likelihood 
\begin{align*}
	L(\theta) =  -\log( P(Y_X; \theta) ),
\end{align*}
where $P(Y_X; \theta)$ is as in \cref{eq:copula_margin_density}, but we now explicitly annotate the dependence on $\theta$. A common approach to minimize this non-convex function is to use conjugate gradient optimization with random restarts or simulated annealing. This requires numerous evaluations of $L$. As it can be seen from \cref{eq:copula_margin_density}, the biggest computational costs are introduced by $\mathcal{N}_{0,K(X,X)}(y)$ from the numerator in \cref{eq:copula_density} which is given by
\begin{align*}
\frac{1}{(2\pi)^{n/2} \vert K(X,X) \vert} \exp \left( -\frac{1}{2} y^T K(X,X)^{-1} y \right),
\end{align*} 
where costs are dominated by the inversion of $K(X,X)$ with $\mathcal{O}(n^3)$. Especially for multi-task problems this becomes rapidly troublesome. Recall, that $n$ is the number of elements in $X$ and $y$ and that, in the case of multi-task learning, we collected the data from all tasks in these two sets. If we assume that each of our $t$ tasks has roughly the same number of training examples, say $\bar{n}$, then the complexity is $\mathcal{O}((t\bar{n})^3)$. This makes multi-task learning computationally very difficult and we introduce approximation scheme next to attack this problem.

\section{Transductive Multi-Task Learning}
\label{sec:transductive_mt}
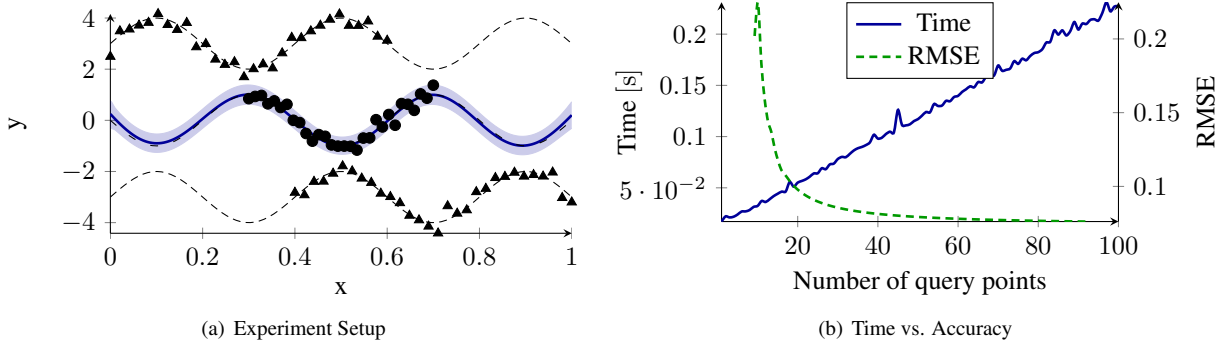
\begin{figure*}[t] 
	\centering
	\subfloat[Experiment Setup]{\label{fig:toy_transductive_training}
	\begin{tikzpicture}
		\begin{axis}[width=.45\textwidth, height=4.5cm, axis x line=bottom, axis y line=left,xlabel=x,ylabel=y]	
		\def\filename{demo_transductive_toy_trans_pred.txt}
		
		\pgfplotstableread{\filename}\ubtable
		\pgfplotstablecreatecol[create col/expr={\thisrow{3}}]{bound}\ubtable
		\pgfplotstablesort[sort cmp=float <]{\ubtable}{\ubtable}
		
		\pgfplotstableread{\filename}\table
		\pgfplotstablecreatecol[create col/expr={\thisrow{2}}]{bound}\table
		\pgfplotstablesort[sort cmp=float >]{\table}{\table}
		
		\pgfplotstablevertcat{\table}{\ubtable}		
		
		\addplot[smooth,darkblue,line width=1pt] table[x index=0,y index=1,] {\filename};	
		\addplot[smooth,darkblue,fill=darkblue,opacity=0.2] table[x index=0, y = bound]{\table};
		\addplot gnuplot[loosely dashed, no marks,black,domain=0:1,samples=100]{-sin(5*pi*x)};		
		\addplot gnuplot[densely dashed,no marks,black,domain=0:1,samples=100]{sin(5*pi*x)+3};
		\addplot gnuplot[densely dashed,no marks,black,domain=0:1,samples=100]{sin(5*pi*x)-3};		
		\addplot[only marks,mark=*,mark options={fill=black}] table[x index=0, y index=3] {demo_transductive_toy_training_data.txt};	
		\addplot[only marks,mark=triangle*, mark options={fill=black}] table[x index=1, y index=4] {demo_transductive_toy_training_data.txt};	
		\addplot[only marks,mark=triangle*, mark options={fill=black}] table[x index=2, y index=5] {demo_transductive_toy_training_data.txt};			
		\end{axis}
	\end{tikzpicture}
	}
	\subfloat[Time vs. Accuracy]{\label{fig:toy_time_acc}
	\begin{tikzpicture}
	\begin{axis}[width=.40\textwidth, height=4.5cm, axis x line=bottom, axis y line=left,xlabel=Number of query points,ylabel=Time,legend style={at={(0.5,1)}, anchor=north},y unit = s,]
		\addplot[smooth,darkblue,line width=1pt] table[x index=0,y index=2,] {demo_transductive_toy_time_accuracy.txt}; \addlegendentry{Time}
		\addplot[densely dashed,smooth,darkgreen,line width=1pt]coordinates {(1,0.13739729)}; 
		\addlegendentry{RMSE}
	\end{axis}
	
	\begin{axis}[width=.40\textwidth, height=4.5cm, axis x line=none, axis y line=right,ylabel=RMSE,axis y line*=right,]
		\addplot[densely dashed,smooth,darkgreen,line width=1pt] table[x index=0,y index=4,] {demo_transductive_toy_time_accuracy.txt}; 
	\end{axis}	
	\end{tikzpicture}
	}
	\caption{a) shows the setup of the toy example. The primary variable to estimate is the black dashed line in the middle and the black circles are the noisy samples. Above and below are the co-variables and their samples (black triangles). The copula median estimate for the primary variable is the blue solid line in the middle nearly on top of the primary variable. The shaded region denotes the area between the $0.05$ and $0.95$ quantiles. b) shows how the prediction time and the estimation error behaves if the number of query points gets enlarged. It can be seen, that the predictions getting better with more query points. This is a property of transductive algorithms.}
	\label{fig:toy_example}
\end{figure*}

As mentioned in the previous section, many learning algorithms, such as the ones we used in this work, can only handle a limited number training data efficiently. This makes it even harder to apply to multi-task problems, since each task carries additional data. In Kriging, Gaussian processes and Gaussian copula processes we have to do a covariance (kernel) matrix inversion, which scales cubic with the number of training data. In this section we present a transductive approach for multi-task algorithms inspired by the Bayesian committee machine \cite{Tresp:2000ba}.

Informally speaking, we are going to perform multi-task learning with the primary variable of interest and each of the secondary variables individually and combine the results at the end. This will reduce the computational costs to $\mathcal{O}(t\bar{n}^3)$.
\begin{thm}\label{thm:transductive}
Let $Y_{X_1},\dotsc,Y_{X_t}$ be the random variables modeling each of the $t$ tasks and we assume without the loss of generality that we want to make predictions for the primary variable $Y_{X^\ast_1}$ for task $1$. Using the assumption that any two $Y_{X_i}, Y_{X_j}$ with $i \neq j \in \{2\dotsc,t\}$ are conditionally independent given $Y_{X_1}$ and $Y_{X^\ast_1}$, we can approximate the full multi-task model as
\begin{align*}
\Pr(Y_{X^\ast_1} \vert Y_{X_1}, \dotsc, Y_{X_t}) \approxeq \frac{\prod_{i=2}^t \Pr(Y_{X^\ast_1} \vert  Y_{X_1}, Y_{X_i}) }{\Pr(Y_{X^\ast_1} \vert Y_{X_1})^{t-2}} \cdot \text{const}.
\end{align*}
\end{thm}

\begin{proof}
With the help of the Bayes' rule and chain rule we write 
\begin{align*}
\Pr(&Y_{X^\ast_1} \vert Y_{X_1}, \dotsc, Y_{X_t}) \\ &= \frac{\Pr(Y_{X^\ast_1}) \Pr(Y_{X_1}\vert Y_{X^\ast_1}) \cdots \Pr( Y_{X_t} \vert Y_{X_1}, \dotsc, Y_{X_{t-1}},Y_{X^\ast_1})}{\Pr(Y_{X_1}, \dotsc, Y_{X_t})},
\end{align*}
which is true in general. In the next step we are making our independence assumptions which yields
\begin{align*}
\Pr(& Y_{X^\ast_1}  \vert Y_{X_1}, \dotsc, Y_{X_t}) \\ 
&\approxeq \frac{\Pr(Y_{X^\ast_1}) \Pr(Y_{X_1},  \vert Y_{X^\ast_1}) \prod_{i=2}^t\Pr( Y_{X_i} \vert Y_{X_1}  ,Y_{X^\ast_1})}{\Pr(Y_{X_1}, \dotsc, Y_{X_t})} \\
&= \frac{\Pr(Y_{X^\ast_1}) \Pr(Y_{X_1}\vert Y_{X^\ast_1})^{t-1} \prod_{i=2}^t\Pr(Y_{X_i} \vert Y_{X_1},Y_{X^\ast_1})}{\Pr(Y_{X_1}, \dotsc, Y_{X_t})\Pr(Y_{X_1}\vert Y_{X^\ast_1})^{t-2}}\\
&= \frac{\Pr(Y_{X^\ast_1}) \prod_{i=2}^t\Pr(Y_{X_i}, Y_{X_1}\vert Y_{X^\ast_1})}{\Pr(Y_{X_1}, \dotsc, Y_{X_t})\Pr(Y_{X_1}\vert Y_{X^\ast_1})^{t-2}}\\
&= \frac{\prod_{i=2}^t \Pr(Y_{X^\ast_1} \vert  Y_{X_1}, Y_{X_i}) }{\Pr(Y_{X^\ast_1} \vert Y_{X_1})^{t-2}} \cdot \text{const},
\end{align*}
where we multiplied in the second step the whole equation with $\Pr(Y_{X_1}   \vert Y_{X^\ast_1})^{t-1}/\Pr(Y_{X_1}   \vert Y_{X^\ast_1})^{t-1}$ and used Bayes' rule again in the last step.
\end{proof}
Notice, that with this approximation, we never have to learn a model for more than two tasks at a time, which gives the computational speedup and also provides a way to easily distribute the computation to several machines.

If we apply the approximation to Gaussian copula processes, the numerator and denominator are conditional Gaussian copula densities of the form as in \cref{eq:gcpdensity}. This is advantageous since we only have to deal with products and quotients of Gaussian distributions introduced by \cref{eq:copula_density}, for which analytical solutions are available. More precisely, the approximate predictive distribution for the Gaussian copula process is then
\begin{align*}
\Pr(Y_{X^\ast_1} \vert Y_{X_1}, \dotsc, Y_{X_t}) & \approxeq c_{\dot{\mu},\dot{\Gamma}}(F_1^\ast(y^\ast_1),\dotsc,F^\ast_m(y^\ast_m)) \\
& \cdot \prod_{i=1}^m \frac{\partial F^\ast_{i}(y^\ast_i)}{\partial y^\ast_i},
\end{align*}
where $\dot{\mu}$ and $\dot{\Gamma}$ can be obtained from 
\begin{align}
\mathcal{N}_{\dot{\mu},\dot{\Gamma}} &= \prod_{t=2}^t \frac{\mathcal{N}_{\hat{\mu}_{1,i},\hat{\Gamma}_{1,i}}}{\mathcal{N}_{\hat{\mu}_{1,i},\diag(\hat{\Gamma}_{1,i})}}
\left( \frac{\mathcal{N}_{\hat{\mu}_1,\diag(\hat{\Gamma}_1)}}{\mathcal{N}_{\hat{\mu}_1,\hat{\Gamma}_1}} \right)^{t-2},\label{eq:trans_gauss_prod_quot}
\end{align}
and $\hat{\mu}_1,\hat{\Gamma}_1$, $\hat{\mu}_{1,i},\hat{\Gamma}_{1,i}$ are defined as in \cref{eq:gcpdensity} if we calculate the predictive distribution for $\Pr(Y_{X^\ast_1} \vert Y_{X_1})$ and $\Pr(Y_{X^\ast_1} \vert  Y_{X_1}, Y_{X_i})$ respectively. For example $\hat{\Gamma}_{1,i}$ would be obtained as
\begin{align*}
\hat{\Gamma}_{1,i} &= K(X^\ast_1,X^\ast_1) - K([X_1, X_i],X^\ast_1)^T \\
& \cdot K([X_1, X_i],[X_1, X_i])^{-1} K([X_1, X_i],X^\ast_1),
\end{align*}
which is also the main contributor to the complexity of $\mathcal{O}(8(t-1)\bar{n}^3) = \mathcal{O}(t\bar{n}^3)$.
Eq. \ref{eq:trans_gauss_prod_quot} above can be further reduced with the rules for products and quotients of Gaussian distributions which can be found in standard textbooks and in \cite{Tresp:2000ba}, but we omit it here due to paucity of space. Please note also that all $y^\ast_1,\dotsc,y^\ast_m$ are from the primary task and so are their univariate marginal distributions $F^\ast_1,\dotsc,F^\ast_m$.

If we follow \citep{Candela:2005fk}, we can also see our transductive approximation as an inducing approach, where the so called \emph{inducing variables} are defined to be $Y_{X_1}$ and $Y_{X^\ast_1}$. Using this point of view, it may be easier to see that the quality of prediction can depend on the number of query points $Y_{X^\ast_1}$ used. As in general for transductive algorithms, the prediction becomes better, the more query points are used (see \cref{fig:toy_example}). As a consequence, even if only a few estimations are needed, one should include artificial dummy test inputs in the prediction step and then discard them. In most cases this is not a serious problem, since the training/parameter estimation phase is the one, which takes an order of magnitude more time than the prediction phase.

\section{Experiments}
\label{sec:experiments}

The experiments in this section share the following setup: In the first step, we train a standard Gaussian copula process for each task individually and determine which combination of kernel and marginal function is the best for each specific task. This procedure yields good results in general. We use either the squared exponential or the Mat\'{e}rn kernel and take the convolutional multi-task kernel approach to calculate the cross-kernels. For the marginal distribution functions we choose one from the following set: The \emph{normal, log-normal, exponential, generalized extreme value (GEV), gamma, t}-distribution or a \emph{parzen window} estimator. In the second step we use maximum likelihood on the full multi-task Gaussian copula process to optimize the parameters for the kernel and the marginal distributions all-together. 

\subsection{A Toy Dataset}

The experiment is done on an artificial dataset to demonstrate the methodology. The dataset consists of three highly dependent tasks (see \cref{fig:toy_transductive_training}), where only noisy samples of the true (latent) functions are available to the algorithm. The transductive Gaussian copula process is used as described in \label{thm:transductive} and can be summarized as follows: First, we learn two distinct multi-task copula processes, the primary variable (middle) together with each of the secondary variables (top and bottom). In the second step we merge the results again to obtain a single multi-task copula process. In this toy example, the full model is omitted, since it is nearly identical to the approximation.

As mentioned in the previous section, the performance of transductive algorithms can depend on the number and location of the query inputs. \cref{fig:toy_time_acc} illustrates this property. We slowly increasing the number of query inputs, starting with only one and calculate the root-mean-squared-error (RMSE) and the time needed for the prediction averaged over 20 trials. The prediction time increases as expected with the number of query points while the prediction accuracy improves. As we will see, the choice of query points does not have a significant impact on the other two datasets. 

\begin{table}
	\centering
	\begin{tabularx}{1\linewidth}{ l c c }
		\hline
					& 	Cd	[Ni, Zn]			& 	Cu	[Pb,Ni,Zn]	\\
		\hline
		StGCP	&	$0.56\pm 0.07$	&	$14.43\pm 2.62$			\\
		MtGCP	&	$\mathbf{0.42\pm 0.06}$	&	$\mathbf{6.57\pm 1.04}$	\\
		TransGCP 	& 	$0.44\pm 0.09$			& 	$6.96\pm 1.43$\\ \hline
		CK			&	$0.51$			&	$7.8$			\\
		StGP		&	$0.57$			&	$15.8$			\\
		MtGP		&	$0.44$			&	$7.5$			\\
		D200		&	$\sim0.46$			&	-				\\
		F359		&	$\sim0.47$			&	-				\\
		P200		&	$\sim0.45$			&	-				\\
	\end{tabularx}	
	\caption[Jura Comparison]{A comparison between various algorithms for the elements Cd and Cu. The table shows the mean of the absolute error and corresponding standard deviation. The first three rows are our implementation of the single-task Gaussian copula process (StGCP), multi-task Gaussian copula process (MtGCP) and the transductive approximation (TransGCP). The other rows are the numbers for the single-task Gaussian process (StGP) and multi-task Gaussian process (MtGP) from \citep{Alvarez:2008ve} and co-Kriging (CK) from \citep{Bivand:2008ly}. The last three entries are from \citep[Fig. 8]{Alvarez:2011fj} and are the DTC, FITC and PITC approximations with 200, 359 and 200 inducing points respectively.}
	\label{fig:jura_comp}
\end{table}

\begin{table}
	\centering
	\begin{tabularx}{1\columnwidth}{ l l l}
		\hline
		 			& 	Opt. Time 	&	Time/Eval.\\ \hline
		MtGCP Cd [Ni, Zn]	&	$898$ s						&	$0.517$ s \\
		TransGCP Cd [Ni, Zn]	&	$429$ s						&	$0.363$ s\\
		MtGCP Cu	[Pb,Ni,Zn]	&	$1046$ s					&	$0.625$ s\\
		TransGCP Cu [Pb,Ni,Zn]	&	$621$ s						&	$0.409$ s \\ \hline
		D200 Cd [Ni, Zn]		&	$185$ s						&	-		\\
		F359 Cd [Ni, Zn]		&	$691$ s						&	-	\\
		P200 Cd [Ni, Zn]		&	$385$ s						&	-		
	\end{tabularx}
	\caption{The table shows the comparison between the full multi-task copula process (MtGCP) and the transductive approximation (TransGCP) for Cadmium (Cd) and Copper (Cu). The first column indicates the algorithm followed by the primary variable and the secondary variables in brackets. The second column shows the total time needed for the marginal likelihood optimization (Opt. Time) and the last column show the time needed per marginal likelihood function evaluation (Time/Eval). The last three entries are from \citep[Table 8]{Alvarez:2011fj} and the algorithm did not run on the same machine as our results. We just provide the figures for completeness and a rough baseline.}
	\label{fig:jura_cmp_trans_full}
\end{table}

\subsection{The Jura Dataset}

The second experiment is performed on the Jura dataset which contains 359 samples of two categorical variables (land uses and rock type) and the concentration of seven chemical elements (Cadmium, Cobalt, Chromium, Copper, Nickel, Lead and Zinc) from a 14.5 \unit{km^2} region of the Swiss Jura. As in the previous experiment the primary variable has fewer samples than the secondary variables. This can occur in real datasets if, for example, the concentration of one element is harder or more expensive to estimate or the dataset contains missing values. For comparison reasons we use exactly the same setup as in  \citep{Bivand:2008ly,Alvarez:2008ve}: 
\begin{itemize}
\item the dataset is divided into 259 training samples and 100 test samples for the primary variable, but all 359 samples are used for the secondary variables; 
\item for Cadmium (Cd) as the primary variable, the secondary variables are Nickel (Ni) and Zinc (Zn); 
\item for Copper (Cu) as the primary variable, the secondary variables are Lead (Pb), Nickel (Ni) and Zinc (Zn). 
\end{itemize}
Furthermore we are using the Mat\'{e}rn kernel for Cd, Ni and Cu and the squared exponential kernel for Zn and Pb. We are modeling the marginal distribution functions for Cd, Ni and Cu with a generalized extreme value distribution and for Zn and Pb a Gamma distribution is used.

We compare the \emph{mean absolute error} (MAE) for various algorithms in \cref{fig:jura_comp} and show the comparison between the full multi-task copula process and the transductive approximation in \cref{fig:jura_cmp_trans_full}. The number of test inputs did not have a significant influence (less than $1\%$) on the prediction results and therefore the numbers were omitted.

We also included results from \citep{Alvarez:2008ve,Alvarez:2011fj} showing the approximation for convolved multi-task Gaussian processes. Note that the convolved Gaussian processes approximation uses less inducing inputs and is therefore not expected to be as good as our transductive approach. Furthermore we did not run the Gaussian process approximations on the same machine and therefore the runtime (in seconds) is not directly comparable, but were included for completeness.

\subsection{The Concrete Slump Dataset}

\begin{table}
	\centering
	\begin{tabularx}{1\columnwidth}{l l l}
		\hline
		 			& 	MtGCP 	&	TransGCP \\ \hline
		RMSE		&	$5.65\pm 2.15$	&		$5.47\pm 1.89$ \\
		MAE			&	$4.08\pm 1.47$	&		$3.97\pm 1.29$ \\
		Opt. Time	&	$386 \pm 150$ s	&		$320\pm 105$	s \\
		Time/Eval	&	$0.18 \pm 0.09$ s&		$0.09\pm 0.06$	s \\

	\end{tabularx}
	\caption{The table shows the comparison between the full multi-task copula process (MtGCP) and the transductive approximation (TransGCP) for the slump dataset. The rows are the \emph{root mean squared error} (RMSE), \emph{mean absolute error} (MAE), total time needed for the marginal likelihood optimization (Opt. Time) and the last row shows the time needed per marginal likelihood function evaluation (Time/Eval). All results are averaged over 100 trials $\pm$ standard deviation.}
	\label{fig:slump_cmp_trans_full}
\end{table}

Our last experiment is performed on the concrete slump dataset with 103 data points, 7 input variables (cement, slag, fly ash, water, SP, Coarse Aggr. and Fine Aggr.) and 3 output variables (slump (cm), flow (cm)	and 28-day compressive strength (mpa)). The goal is to estimate the variable slump with flow and compressive strength as secondary variables. We split the dataset randomly into 83 training and 20 evaluation points and calculate the average over 100 runs (see \cref{fig:slump_cmp_trans_full}). We found that the combination of Mat\'{e}rn kernel and the generalized extreme value distribution yielded the best results for all three variables.
Interestingly the predictions of the transductive approximation are (on average) better than the predictions of the full model. This can happen since the optimization problem for the transductive Gaussian copula process is slightly easier than the one from the full model and the optimizer can sometimes find a better solution. As for the Jura dataset we did not found any significant changes in predictions if we vary the number of query points.

\section{Conclusions}
\label{sec:conclusions}
This work introduced a new transductive approximation methodology for multi-task learning to solve the computational challenges in copula processes. Copula processes are extremely useful in cases where the assumptions of Gaussian processes are invalid. They allow a different marginal distribution for each query variable while capturing the inter-task dependencies. We showed how the Gaussian copula process framework could be extended to multi-task learning with appropriate kernels and addressing computational challenges. We derived closed-form expressions for the transductive approximation which reduces the computational costs from $\mathcal{O}((t\bar{n})^3)$ to $\mathcal{O}(t\bar{n}^3)$. Furthermore we investigate experimentally, on one synthetic, and two real public datasets, the different properties of the transductive learning approach.

\subsection{Acknowledgments.}
This work was supported by the Rio Tinto Centre for Mine Automation and the Australian Centre for Field Robotics.

\bibliography{literature}

\end{document}